\documentclass{article}

\usepackage{fullpage}

\usepackage[utf8]{inputenc} 
\usepackage[T1]{fontenc}    
\usepackage{hyperref}       
\usepackage{url}            
\usepackage{booktabs}       
\usepackage{amsfonts}       
\usepackage{nicefrac}       
\usepackage{microtype}      
\usepackage{xcolor}         
\usepackage{natbib}

\usepackage{amsmath, mathrsfs}

\usepackage{tikz}

\usepackage{multirow}

\usepackage{algorithm, algpseudocode}

\usepackage{amsthm}

\newtheorem{theorem}{Theorem}
\newtheorem{lemma}[theorem]{Lemma}
\newtheorem{proposition}[theorem]{Proposition}

\newtheorem{corollary}[theorem]{Corollary}
\newtheorem{definition}{Definition}

\newcommand\E{\mathbb{E}}

\newcommand{\cS}{\mathcal{S}}
\newcommand{\cA}{\mathcal{A}}

\title{
Near Sample-Optimal Reduction-based Policy Learning for Average Reward MDP 
}

{\small\author{%
  Jinghan Wang\\
  Peking University\\
  \texttt{wangjinghan@pku.edu.cn} 
   \and
   Mengdi Wang \\
   Princeton University \\
   \texttt{mengdiw@princeton.edu}
   \and 
   Lin F.~Yang \\
   University of California, Los Angeles \\
   \texttt{linyang@ee.ucla.edu}
}}

\begin{document}

\maketitle

\begin{abstract}
This work considers the sample complexity of obtaining an $\varepsilon$-optimal policy in an average reward Markov Decision Process (AMDP), given  access to a generative model (simulator). 
When the ground-truth MDP is weakly communicating, we prove an upper bound of $\widetilde O(H \varepsilon^{-3} \ln \frac{1}{\delta})$ samples per state-action pair, where $H := sp(h^*)$ is the span of bias of any optimal policy, $\varepsilon$ is the accuracy and $\delta$ is the failure probability.
This bound improves the best-known mixing-time-based approaches in \cite{jin2021towards}, which assume the mixing-time of every deterministic policy is bounded.
The core of our analysis is a proper reduction bound from AMDP problems to discounted MDP (DMDP) problems, which may be of  independent interests since it allows the application of DMDP algorithms for AMDP in other settings. We complement our upper bound by proving a minimax lower bound of $\Omega(|\mathcal S| |\mathcal A| H \varepsilon^{-2} \ln \frac{1}{\delta})$ total samples, showing that a linear dependent on $H$ is necessary and that our upper bound matches the lower bound in all parameters of $(|\mathcal S|, |\mathcal A|, H, \ln \frac{1}{\delta})$ up to some logarithmic factors.
\end{abstract}

\section{Introduction}
\label{sec:intro}
In this paper we study the fundamental problem of learning an optimal policy in a
Markov decision processes (MDP), which is the most fundamental model for reinforcement learning (RL).
In this model, there is a finite set of states, $\cS$, for each of which, there is a set of actions $\cA$. An agent is able to interact with the MDP by playing actions and then the environment will transition to a new state based on some fixed but unknown probability and at the same time return a scalar reward to the agent.
For different variants of the MDP, the agent aims to maximize different objectives. For instance, in the discounted setting (DMDP), the agent aims to maximize the expected cumulative rewards, $\mathbb{E}\left[\sum_{t=0}^\infty \gamma^t r_t\right]$, where $r_t$ is the reward collected at time $t$ and $\gamma \in [0,1)$ is a discount factor, which avoids divergence in the infinite sum. Another important variant is the average reward MDP (AMDP), where the objective is to maximize the average cumulative rewards: $\underset{T\to \infty}{\lim}\mathbb{E}\left[\frac{1}{T}\sum_{t=0}^{T-1}  r_t\right]$. As $\gamma <1$, the DMDP value is roughly dominated by its first $(1-\gamma)^{-1}$ terms in the sum, whereas the AMDP is considering an infinite horizon. Both variants are of great importance in decision making and machine learning research (\citet{bertsekas1995neuro}, \citet{puterman2014markov}, \citet{sutton2018reinforcement}). 

In the learning setting, the transition probability is unknown and hence an agent needs to interact with the MDP to collect data and solve the MDP. In order to characterize the learning sample complexity, we study the problem with the fundamental generative model \citet{kearns1998finite}, in which the agent is able to query each state-action pair for arbitrary number of samples (of the next state and the reward). Then the goal is to output a good stationary and deterministic policy $\pi:\cS\rightarrow \cA$\footnote{As shown in \citet{puterman2014markov}, the optimal policy can always be stationary and deterministic.} such that the value $\rho^{\pi}:=\underset{T\to \infty}{\lim}\mathbb{E}\left[\frac{1}{T}\sum_{t=0}^{T-1}  r_t\right]$ is approximately maximized, where $r_t$ is the reward collected at time $t$ by following the actions of $\pi$. 

Recently, there is a line of works \citet{kearns1998finite}, \citet{kakade2003sample}, \citet{gheshlaghi2013minimax}, \citet{jin2018q}, \citet{sidford2018near}, \citet{sidford2018variance}, \citet{wainwright2019variance}, \citet{tu2019gap}, \citet{agarwal2020model}, \citet{li2020breaking}, settling the learning complexity of DMDP with generative model. 
The state-of-the-art sample upper bound is provided in \citet{li2020breaking}: $\widetilde{O}(|\cS||\cA|(1-\gamma)^{-3}\varepsilon^{-2})$\footnote{$\widetilde{O}(\cdot)$ hides log factors} for an $\varepsilon$-optimal policy. A matching lower bound is provided in \citet{gheshlaghi2013minimax}.
However, for AMDP, 
despite a number of recent advances \citet{wang2017primal}, \citet{jin2020efficiently}, \citet{jin2021towards},
similar characterizations have not been achieved yet. The tightest upper bound so far is $\widetilde{O}(t_{\rm mix}\varepsilon^{-3})$ by \citet{jin2021towards}, where $t_{\rm mix}$ is the \emph{worst-case} mixing time for a policy (i.e., the mixing time of the Markov chain induced by playing a policy). They also showed a minimax lower bound $\Omega(t_{\rm mix}\varepsilon^{-2})$. 
Yet, whether the characterization using the worst-case mixing time is stringent remains illusive. One may simply question: why the worst policy can affect the complexity of learning the best policy?

In the online setting (no generative model), however, \citet{auer2008near} shows that the learning complexity (regret) is depending on a notion called the diameter, which measures the \emph{best} hitting time from one state to another. Such a diameter can be small even if there are very bad policies in the MDP. Recently, the regret bound was additionally improved in \citet{zhang2019regret} by substituting the diameter $D$ in the upper bound into another even smaller parameter $H$, the span\footnote{The span of a vector $x$ is defined to be $sp(x) := \max_i x_i - \min_i x_i$.} of bias of a gain-optimal policy. Nevertheless, as these online algorithms only produce \emph{non-stationary} policies\footnote{Note that standard techniques (e.g.\cite{jin2018q}) converting an algorithm with small regret bound to a small PAC-bound only generates a non-stationary policy for the AMDP setting.}, they cannot be applied in our setting. 
To best of our knowledge, we are not aware of any algorithms having a sample complexity only depending on the diameter $D$ or the bias $H$ of the AMDP. A list of recent results are presented in Table~\ref{table:res-MDP}. 

In this paper, we revisit the learning question in AMDP with a generative model and achieve nearly tight complexity characterizations
. In particular, we show that if an AMDP is weakly communicating and an upper bound of $H$, the span of bias of any optimal policy, is known, there is an algorithm that learns an $\varepsilon$-optimal policy with $\widetilde{O}(|\cS||\cA|H\varepsilon^{-3})$ samples. We complement our upper bound by a nearly matching lower bound $\widetilde{\Omega}(|\cS||\cA|H\varepsilon^{-2})$. 
Our result improves the result in \citet{jin2021towards}, where the bound depends on $t_{mix}$, since $t_{mix}$ is always larger than $H$ (see Appendix~\ref{sec:mix}) and can be considerably larger in certain MDPs.

\subsection{Related works}
There is long line of research targeting the sample complexity of DMDPs with a generative model. 
On the lower bound side,
\citet{gheshlaghi2013minimax} proves a minimax lower bound of $N = \widetilde \Omega ( |\mathcal S| |\mathcal A| (1-\gamma)^{-3} \varepsilon^{-2} )$ to promise to find an $\varepsilon$-optimal policy. On the upper bound side, \citet{gheshlaghi2013minimax} first shows that, for any $\varepsilon \in (0, 1)$, within $N = \widetilde O (|\mathcal S| |\mathcal A| (1-\gamma)^{-3} \varepsilon^{-2})$ samples the model-based approaches can estimate the optimal Q-function to an $\varepsilon$-accuracy, and thus obtains an upper bound of $N = \widetilde O (|\mathcal S| |\mathcal A| (1-\gamma)^{-5} \varepsilon^{-2})$ by using the greedy policy in this estimated Q-function. This work also proves an tight upper bound of $N = \widetilde O (|\mathcal S| |\mathcal A| (1-\gamma)^{-3} \varepsilon^{-2})$ for a small range of accuracy, $\varepsilon \in (0, (1-\gamma)^{-1/2} |\mathcal S|^{-1/2})$. The $\varepsilon$ range is then gradually expanded in a sequence of following works: to $(0, 1)$ in \citet{sidford2018near} using a variance-reduced Q-value iteration and in \citet{wainwright2019variance} using a variance-reduced Q-learning; to $(0, (1-\gamma)^{-1/2})$ in \citet{agarwal2020model} using the definition of a leave-one-out absorbing MDP; and finally to the full range $(0, (1-\gamma)^{-1})$ in \citet{li2020breaking}.

AMDPs are less well-studied in terms of sample complexity. The existing bounds are mostly based on a parameter $t_{\text{mix}}$, which is the maximal mixing time (to $\varepsilon$ accuray) of all stationary and deterministic policies. \citet{wang2017primal} first gives an upper bound of $N = O( \tau^2 |\mathcal S| |\mathcal A| t_{\text{mix}}^2 \varepsilon^{-2} )$, where $\tau$ is an upper bound of the ergodicity of all invariant distribution of all stationary policies, by a primal-dual method to the linear programming version of AMDP problems. The ergodicity assumption is removed in \citet{jin2020efficiently} by using a primal-dual stochastic mirror descent method, and the bound is improved to $N = O( |\mathcal S| |\mathcal A| t_{\text{mix}}^2 \varepsilon^{-2} )$. Recently, \citet{jin2021towards} further improves the upper bound to $N = \widetilde O( |\mathcal S| |\mathcal A| t_{\text{mix}} \varepsilon^{-3} )$ by reducing AMDPs to DMDPs, and also proves a lower bound of $N = \Omega( |\mathcal S| |\mathcal A| t_{\text{mix}} \varepsilon^{-2} )$, showing a tight dependence on $t_{\text{mix}}$.

When it comes to the online setting, where we pursue a low total regret instead of an $\varepsilon$-optimal stationary policy, 
\citet{auer2008near} first gives an online algorithm with $T$-step regret upperbounded by $\Delta(T) = \widetilde O(D |\mathcal S| \sqrt{|\mathcal A| T})$ where $D$ is the diameter of the MDP, and also proves a lower bound $\Delta(T) = \Omega (\sqrt{|\mathcal S| |\mathcal A| DT})$. The most recent progress, \citet{zhang2019regret}, improves the upper bound to $\Delta(T) = \widetilde O(\sqrt{|\mathcal S| |\mathcal A| HT})$, where $H := sp(h^*) \leq D$ is the span of optimal policy (the difference of the maximum and minimum value in the bias function $h^*$), matching the lower bound. 
Unfortunately, regret bounds do not directly imply $\varepsilon$-optimal stationary policies in the generative model. Hence these bounds are not comparable with ours.

\subsection{Approach}

Our algorithm and upper bound is based on the well-known fact that an AMDP is the limiting case of a DMDP as $\gamma \rightarrow 1$. Especially, the Laurent series expansion (\citet[Sec 8.2.2]{puterman2014markov}) of $V_\gamma^\pi$ about a new parameter $\beta := (1 - \gamma) \gamma^{-1}$ reveals that $V_\gamma^\pi \sim (1 - \gamma) \rho^\pi$ as $\gamma \rightarrow 1^-$. However, this expansion is restricted only for $\gamma$ sufficiently near $1$ and does not hold for all $\gamma \in (0, 1)$, making it an unsatisfying tool to analyse the reduction. Previous work (\citet{jin2021towards}) instead establishes reduction from AMDPs to DMDPs under the global mixing assumption and relies on the global parameter $t_{mix}$, by showing that $\| \rho^\pi - (1 - \gamma) V_\gamma^\pi \|_\infty = O((1 - \gamma) t_{mix})$ holds for all policies.

To avoid the dependence of upper bound on $t_{mix}$-style parameters, we first prove an intermediate inequality, Lemma \ref{lemma:reduction}, which states that $\| \rho^\pi - (1 - \gamma) V_\gamma^\pi \|_\infty \leq sp((1 - \gamma) V_\gamma^\pi)$ for any policy and any $\gamma$, to bypass the analysis for all policies then. To reduce from AMDPs to DMDPs, we only need to bound the RHS in Lemma \ref{lemma:reduction} for two special policies: $\pi = \pi^*$, the optimal policy in the AMDP; and $\pi = \hat\pi$, a near-optimal policy in the DMDP. Furthermore, because of the near-optimality of $\hat\pi$, $V_\gamma^{\hat\pi}$ and $V_\gamma^*$ have to be close, and so are their spans. Then the remaining is to bound $sp(V_\gamma^\pi)$ just for two optimal policies, $\pi^*$ and $\pi_\gamma^*$, not again requiring a bound for all policies. Some further careful analyses show that $\gamma = 1 - \Theta(\frac{\varepsilon}{H})$ is enough to promise an upper bound of $O(\varepsilon)$ for RHS in Lemma \ref{lemma:reduction}, and then for $\| \rho^* - \rho^{\hat\pi} \|_\infty$.

With the novel reduction bound above, our algorithm takes $\gamma = 1 - \Theta(\frac{\varepsilon}{H})$, calls an near-optimal algorithm for DMDPs to solve the reduced DMDP to an accuracy of $\varepsilon_\gamma = O(\frac{\varepsilon}{1-\gamma})$, and outputs its resulting policy. The sample complexity of this composed algorithm is then
\[ \widetilde O \left( \frac{|\mathcal S| |\mathcal A|}{(1 - \gamma)^3 \varepsilon_\gamma^2} \right) = \widetilde O \left( \frac{|\mathcal S| |\mathcal A|}{(1 - \gamma) \varepsilon^2} \right) = \widetilde O \left( \frac{|\mathcal S| |\mathcal A| H}{\varepsilon^3} \right). \]

As for the lower bound, we modify the construction in \citet{auer2008near} and use a similar analysis as in \citet{feng2019does}. In our construction of hard MDPs, $M_1$ and $M_{k,l} (\forall k \forall l)$, $M_1$ disagrees with $M_{k,l}$ on the $\varepsilon$-optimal action choice at a state $x_k$. However, the only difference between $M_1$ and $M_{k,l}$ is the transition probability at the state-action pair $(x_k, a_l)$, and an information theoretical analysis shows that $\Omega(D \varepsilon^{-2} \ln\frac{1}{\delta})$ independent samples are necessary at this state-action pair to distinguish from $M_1$ and $M_{k,l}$ with high probability, where $D$ is the diameter of the constructed MDPs. Sum the samples up for all $M_{k, l}$ and the lower bound of $\Omega (|\mathcal S| |\mathcal A| D \varepsilon^{-2} \ln\frac{1}{\delta})$ follows. Since $D$ is always an upper bound of $H$, a lower bound of $H$, $\Omega (|\mathcal S| |\mathcal A| H \varepsilon^{-2} \ln\frac{1}{\delta})$, also follows.


\begin{table*}[t]
\centering
\caption{Results of DMDPs and AMDPs (online setting not listed)\label{table:res-MDP}}
\begin{tabular}{cccc}
\toprule
Type & Method & Sample Complexity & Accuracy \\
\midrule
\multirow{6}{*}{DMDP}
 & Lower Bound & $\widetilde \Omega ( |\mathcal S| |\mathcal A| (1-\gamma)^{-3} \varepsilon^{-2} )$ & N/A \\
 & Empirical Q-value Iteration & $\widetilde O (|\mathcal S| |\mathcal A| (1-\gamma)^{-5} \varepsilon^{-2})$ & $(0, 1)$ \\
 & Empirical Q-value Iteration & $\widetilde O (|\mathcal S| |\mathcal A| (1-\gamma)^{-3} \varepsilon^{-2})$ & $(0, \frac{1}{\sqrt{(1-\gamma)|\mathcal S|}})$ \\
 & Variance-reduced Q-value Iteration & $\widetilde O (|\mathcal S| |\mathcal A| (1-\gamma)^{-3} \varepsilon^{-2})$ & $(0, 1)$ \\
 & Empirical + Model-based & $\widetilde O (|\mathcal S| |\mathcal A| (1-\gamma)^{-3} \varepsilon^{-2})$ & $(0, \frac{1}{\sqrt{1-\gamma}})$ \\
 & Perturbed Empirical + Model-based & $\widetilde O (|\mathcal S| |\mathcal A| (1-\gamma)^{-3} \varepsilon^{-2})$ & $(0, \frac{1}{1-\gamma})$ \\
\midrule
\multirow{6}{*}{AMDP}
 & Lower Bound ($t_{mix}$) & $\Omega( |\mathcal S| |\mathcal A| t_{mix} \varepsilon^{-2} )$ & N/A \\
 & Primal-Dual & $O( \tau^2 |\mathcal S| |\mathcal A| t_{mix}^2 \varepsilon^{-2} )$ & $(0, 1)$ \\
 & Primal-Dual SMD & $O( |\mathcal S| |\mathcal A| t_{mix}^2 \varepsilon^{-2} )$ & $(0, 1)$ \\
 & Reduction to DMDPs ($t_{mix}$) & $\widetilde O( |\mathcal S| |\mathcal A| t_{mix} \varepsilon^{-3} )$ & $(0, 1)$ \\
 & \textbf{Reduction to DMDPs} ($H$, \textbf{this work}) & $\widetilde O( |\mathcal S| |\mathcal A| H \varepsilon^{-3} )$ & $(0, 1)$ \\
 & \textbf{Lower Bound} ($D$, \textbf{this work}) & $\widetilde \Omega( |\mathcal S| |\mathcal A| D \varepsilon^{-2} )$ & N/A \\
\bottomrule
\end{tabular}
\end{table*}

\section{Setting}
\label{sec:setting}

\paragraph{Markov decision process}
A Markov decision process (MDP) is a tuple $M = (\mathcal S, \mathcal A, P, r)$, where $\mathcal S$ is a finite set of states, $\mathcal A$ is a finite set of actions, $P : \mathcal S \times \mathcal A \rightarrow \mathbb R^\mathcal S$ is the transition probability to any state when taking any action at any state, and $r : \mathcal S \times \mathcal A \rightarrow [0, 1]$ is the reward function. A stationary policy $\pi : \mathcal S \rightarrow \mathbb R^\mathcal A$ maps any state to a probabilistic distribution on $\mathcal A$, where the policy takes action $a$ at state $s$ with probability $\pi(a | s)$. A trajectory is the finite or infinite sequence of states, actions and rewards, denoted as $\{ s_0, a_0, r_0, s_1, a_1, r_1, \cdots \}$. A trajectory can be induced by a stationary policy by setting $a_t \sim \pi(\cdot | s_t), r_t = r(s_t, a_t), s_{t+1} \sim P(\cdot | s_t, a_t)$ and $s_0 \sim p_0(\cdot)$ obeying some initial state distribution. A trajectory can also be induced by a history-dependent policy, who determines the choice of action every step based on the history trajectory. By default, we only consider stationary policies, except when specially noted.

\paragraph{Discounted MDP}
A discounted MDP (DMDP) is a MDP together with a parameter $\gamma \in (0, 1)$ called discounted factor and the discounted value function, defined as:
\[ \forall s \in \mathcal S :\ \  V_\gamma^\pi(s) := \E \left[ \sum_{t=0}^\infty \gamma^t r_t | s_0 = s \right], \]
where the expectation is taken under the distribution of trajectories induced by policy $\pi$. It's obvious that $0 \leq V_\gamma^\pi \leq \frac{1}{1-\gamma}$. The optimal value function is defined as
\[ V_\gamma^* := \underset{\pi}{\max} V_\gamma^\pi, \]
where the maximum is taken in the space of all history-dependent policies and for all states simultaneously. It's a classical fact (\citet[Sec 6.2.4]{puterman2014markov}) that there exists an optimal policy $\pi_\gamma^*$ which is stationary.

\paragraph{Finite horizon MDP}
A finite horizon MDP (FMDP) is a MDP only running $T$ steps and with a value function without discounting. Its value function is defined as
\[ \forall s \in \mathcal S :\ \  V_T^\pi(s) := \E \left[ \sum_{t=0}^{T-1} r_t | s_0 = s \right]. \]

\paragraph{Average reward MDP}
An average reward MDP (AMDP) is a MDP running infinite steps but with an average gain function instead of a discounted value function. The average gain of a policy can be defined as
\[ \forall s \in \mathcal S :\ \  \rho^\pi(s) := \underset{T \rightarrow \infty}{\lim} \frac{1}{T} V_T^\pi(s), \]
and the convergence of the limitation can be shown to hold forever. Another value function called bias function is useful in AMDP, defined as
\[ \forall s \in \mathcal S :\ \  h^\pi(s) := \underset{T \rightarrow \infty}{C\text{-}\lim} \left[ V_T^\pi(s) - T \rho^\pi \right], \]
where the limitation is a Cesaro limitation and can be replaced by an ordinary limitation when the MDP is aperiodic. Again, it can be shown that $h$ is always well-defined. Like in DMDP, we can also define optimal policies in AMDP. There always exists a stationary policy $\pi^*$ which is gain-optimal in all history-dependent policies (\citet[Sec 9.1.3]{puterman2014markov}):
\[ \rho^{\pi^*} = \rho^* := \underset{\pi}{\max}\rho^\pi, \]
and its bias function is denoted by $h^{\pi^*} = h^*$. The span of $h^*$ plays an important role in the analysis of AMDP, and it's famous (\citet{bartlett2012regal}) that $H := sp(h^*) \leq D$.

\paragraph{Generative model}
When considering the sample complexity, we assume we only have access to a generative model \citet{kearns1998finite} of the ground-truth MDP $M = (\mathcal S, \mathcal A, P, r)$, which can only provide us with samples $s^\prime \sim P(\cdot | s, a)$ for any $(s, a)$ pair. In this work, we assume for simplicity that the reward function $r$ is known and deterministic, since it only introduces lower order terms if remove this assumption.

\paragraph{DMDP oracles}
Though deeper comprehension is still waited to be established in AMDP, literature has already understood much about DMDP, including both algorithms and theories. \citet{li2020breaking} shows that their algorithm promises to find an $\varepsilon$-optimal (stationary) policy with probability at least $1 - \delta$ in $\widetilde O( (1-\gamma)^{-3} \varepsilon^{-2} \ln \frac{1}{\delta} )$ samples per state-action pair from a generative model of the underground MDP $M$, for any $\varepsilon \in (0, (1-\gamma)^{-1}]$, matching the sample complexity lower bound up to some logarithmic factors. In this work, we will use this algorithm as an oracle to the DMDP problems.

\paragraph{Connection assumption of MDP}
MDPs can be classified into several types, according to that to what extent the states connect with each other. A reasonable requirement is that the MDP is weakly communicating, which means that, except some states which are always transient in the Markov chain induced by any policy, for any two other states $s_1, s_2$, there is a policy to go from $s_1$ to $s_2$. The property of being weakly communicating can guarantee that the optimal gain $\rho^*(s)$ is a constant function. A stronger connection assumption is requiring the MDP to be communicating, which doesn't allow the existence of the transient states. It's equivalent to $D < +\infty$, where the diameter $D$ of a MDP is defined as
\[ D := \underset{s_1 \neq s_2}{\max} \underset{\pi}{\min}\  \E^\pi_{s_1} [\tau_{s_2}], \]
where $\tau_{s_2}$ denotes the hitting time to $s_2$. 

\section{Main results}
\label{sec:mainres}

In this section we present our main results.
The core of this work is the reduction bound from AMDPs to DMDPs, which enables us to solve AMDP problems using the well-developed DMDP algorithms (oracles). The key challenge is that, to reduce an AMDP problem to a DMDP problem, we have to show a near-optimal policy $\hat\pi$ in DMDP is also near-optimal in the AMDP. For the quantities involved with an exactly optimal policy, e.g. $\pi^*$, we do have some useful tools to analyse them, such as the famous inequality $sp(h^{\pi^*}) \leq D$. But we know little about $\hat \pi$, except that it enjoys a near-optimality property in the reduced DMDP problem. The usual analyses mostly ignore the particularity of $\hat \pi$, simply regard it as an arbitrary policy and establish bounds for arbitrary policies via uniform convergence, which results in upper bounds in term of the global quantities like $t_{mix}$ (\citet{jin2021towards}). However, $t_{mix}$ is a pessimistic parameter, and it can be significantly larger than $H$ in certain MDPs. Moreover, in some MDPs, such as in periodic MDPs, $t_{mix} = +\infty$ and the bound in \citet{jin2021towards} fails.

In this work, we instead prove reduction bounds in form of span of value functions. Such form of bounds have the advantage that they can be passed down the proof from known ones to unknown ones of $\hat \pi$, as will be seen in the following proofs. We assume $H \geq 1$ without loss of generality.

\begin{theorem} \label{thm:reduction}
(Main Result - Reduction Bound from AMDP to DMDP)
Suppose MDP $M$ is weakly communicating. 
For any $\varepsilon \in (0, 1]$, take $\gamma = 1 - \frac{\varepsilon}{H}$. Then for any $\varepsilon_\gamma \in [0, \frac{1}{1-\gamma}]$ and any $\varepsilon_\gamma$-optimal policy $\pi$ in DMDP $(M, \gamma)$, it holds that $\rho^* - \rho^\pi \leq 8 \varepsilon + 3 (1-\gamma) \varepsilon_\gamma = (8 + 3 \frac{\varepsilon_\gamma}{H}) \varepsilon$.
\end{theorem}

Thus, this reduction bound shows that if the DMDP $(M, \gamma)$ is solved to an accuracy of $\varepsilon_\gamma = O(\frac{\varepsilon}{1-\gamma})$, the resulting policy will be of $O(\varepsilon)$ accuracy in the original AMDP. Note that this reduction bound holds for any oracle to the DMDP problems, including both model-free and model-based methods, and is not restricted to the case of generative models.

When it comes to the sample complexity, we can use as oracle the algorithm defined in \citet{li2020breaking}, which is based on a perturbed empirical MDP construction, to realize a near-optimal sample complexity in DMDPs (see Algorithm \ref{alg}). We state the fact of this DMDP oracle in the following lemma:

\begin{lemma} \label{lemma:DMDP}
(\citet[Theorem 1]{li2020breaking})
There exists some universal constant $c_0 > 0$ such that: for any $\delta > 0$, any $\gamma \in (0, 1)$ and any $0 < \varepsilon_\gamma \leq \frac{1}{1-\gamma}$, the policy $\hat\pi$ found by the algorithm 1 in \citet{li2020breaking} obeys $\| V_\gamma^{\hat\pi} - V_\gamma^* \|_\infty \leq \varepsilon_\gamma$ with probability at least $1 - \delta$, provided that the sample size per state-action pair exceeds
\[ N \geq \frac{c_0 \ln (\frac{|\mathcal S| |\mathcal A|}{(1-\gamma)\varepsilon_\gamma\delta})}{(1-\gamma)^3\varepsilon_\gamma^2}. \]
\end{lemma}

\begin{algorithm}[t]
\caption{Reducing to DMDP and Solving with Algorithm 1 in \citet{li2020breaking}}
\label{alg}
\begin{algorithmic}[1]
	\Require A generative model of $M = (\mathcal S, \mathcal A, P, r)$, its optimal bias span $H$ (or an upper bound), accuracy $\varepsilon$, failure probability $\delta$;
	\Ensure An $\varepsilon$-optimal policy of this AMDP problem;
	\State Define $\gamma = 1 - \frac{\varepsilon}{12H}$ and $\varepsilon_\gamma = \frac{\varepsilon}{12(1-\gamma)}$;
	\State Define perturbation size $\xi = \frac{c_p(1-\gamma)\varepsilon_\gamma}{|\mathcal S|^5 |\mathcal A|^5}$ ($c_p$ is a universal constant in Theorem 1 in \citet{li2020breaking}), and independently randomly perturb the reward by \[r_p(s, a) = r(s, a) + \zeta(s, a),\ \ \ \zeta(s, a) \overset{\text{i.i.d.}}{\sim} \text{Unif}(0, \xi);\]
	\State Take $N \geq \tilde c_0 H \varepsilon^{-3} \ln (\frac{|\mathcal S| |\mathcal A|}{\varepsilon\delta})$ samples at each state-action pair, and build an empirical MDP $\widehat M_p = (\mathcal S, \mathcal A, \widehat P, r_p)$, where $\widehat P$ is the empirical transition probability with these samples;
	\State Solve DMDP $(\widehat M_p, \gamma)$ with the Q-value Iteration algorithm, and denote the optimal policy as $\hat\pi$;
	\State Output $\hat\pi$.
\end{algorithmic}
\end{algorithm}

With Theorem \ref{thm:reduction} and Lemma \ref{lemma:DMDP}, the following upper bound of sample complexity of AMDP immediately follows:

\begin{theorem} \label{thm:upperbound}
(Main Result - Upper Bound)
Suppose MDP $M$ is weakly communicating.
There exists some universal constant $\tilde c_0 > 0$ such that: for any $\delta \in (0, 1]$ and any $\varepsilon \in (0, 1]$, take $\gamma = 1 - \frac{\varepsilon}{12 H}$, then the policy $\hat\pi$
found by Algorithm \ref{alg} 
obeys 
$\rho^{\hat\pi} \geq \rho^* - \varepsilon$ with probability at least $1 - \delta$, provided that the sample size per state-action pair exceeds
\[ N \geq \tilde c_0 H \varepsilon^{-3} \ln (\frac{|\mathcal S| |\mathcal A|}{\varepsilon\delta}). \]
\end{theorem}


Lastly, we show that our upper bound is tight in all the parameters $(|\mathcal S|, |\mathcal A|, H, \ln \frac{1}{\delta})$ up to some logarithmic factors, only suffering an extra $\varepsilon$, by stating the following lower bound:

\begin{theorem} \label{thm:lowerbound}
(Main Result - Lower Bound)
For any sufficiently small $\varepsilon, \delta$, any sufficiently large $|\mathcal S|, |\mathcal A|$, and any $D \geq \max\{c_1 \log_{|\mathcal A|}{|\mathcal S|}, c_2\}$ (where $c_1, c_2 > 0$ is some universal constant), for any algorithm promising to return an $\varepsilon$-optimal policy with probability at least $1 - \delta$ on any communicating AMDP problem, there is an MDP with parameters not larger than these values such that the expected total samples on all state-action pairs, when running this algorithm, is at least $c |\mathcal S| |\mathcal A| D \varepsilon^{-2} \ln \frac{1}{\delta}$ (where $c > 0$ is some universal constant). In short, the sample complexity of AMDP problems is $\Omega(|\mathcal S| |\mathcal A| D \varepsilon^{-2} \ln \frac{1}{\delta})$.
\end{theorem}

Since $H \leq D$ holds in any MDP, Corollary \ref{thm:lowerbound-H} follows:

\begin{corollary} \label{thm:lowerbound-H}
(Main Result - Lower Bound)
The sample complexity of AMDP problems is $\Omega(|\mathcal S| |\mathcal A| H \varepsilon^{-2} \ln \frac{1}{\delta})$.
\end{corollary}

Actually, the construction in \citet{auer2008near} does work in essence. However, the lower bound there is in form of regret for online algorithms, while our upper bound is of minimax form. In this work, we provide a similar information theoretical analysis as in \citet{feng2019does} when considering $\varepsilon$-optimal stationary policies, showing that our upper bound is tight in all parameters but $\varepsilon$.

\section{Analysis}
\label{sec:analysis}

In this section, we will give a high-level outline of the proof, as well as the statement and explanation of some lemmas, but deferring the detailed proofs to the appendix.

Before stepping into the analysis, we write down some useful notation for the sake of convenience.

\paragraph{Additional Notation}
Given any stationary policy $\pi$, it induces a Markov chain (MC) with transition matrix $P_\pi(s, s^\prime) := \E [ P(s^\prime | s, a) | a \sim \pi(s) ]$. We also denote this MC as $P_\pi$.
The MC can always be uniquely partitioned into several recurrent blocks, $C_k, k = 1, 2, \cdots K_\pi$. Then $P_\pi$ is a block matrix with the same state (index) partition, and so is the limiting matrix $P_\pi^*$. For any function of states, e.g. $\rho^\pi$ and $V_\gamma^\pi$, the restriction of the function on block $C_k$ is denoted with an extra subscript $k$, such as $\rho_k^\pi$ and $V_{\gamma,k}^\pi$.
It's known that, when the MDP is weakly communicating, $\rho^\pi$ is always a constant vector (\citet[Sec 8.3]{puterman2014markov}). It's also known that $\rho_k^\pi$ is always constant for any recurrent block $C_k$ of $P_\pi$.

\subsection{Upper bound}
\label{sec:upperbound}

Lemma \ref{lemma:reduction} is a lemma establishing the connection between DMDP and AMDP value functions. The form of its bound, i.e. span of value functions, is crucial in our analysis.

\begin{lemma} \label{lemma:reduction}
Suppose MDP $M$ is weakly communicating. For any $\gamma \in (0, 1)$ and any stationary policy $\pi$, 
\[ \|\rho^\pi - (1-\gamma)V_\gamma^\pi\|_\infty \leq sp((1-\gamma)V_\gamma^\pi). \]
\end{lemma}

\begin{proof}[Proof Idea]
This lemma follows directly from the closed-form expressions of $\rho^\pi$ and $V_\gamma^\pi$, in essence saying that, for any $s \in C_k$, $\rho^\pi(s)$ is an average of $V_{\gamma,k}^\pi$. See Appendix \ref{sec:appendix-upper}.
\end{proof}

Remaining is to bound the spans in Lemma \ref{lemma:reduction} for $\pi = \pi^*$ and $\pi_\gamma^*$, and the bound for $\pi =\hat\pi$ can be then passed from $\pi_\gamma^*$ through Lemma \ref{lemma:DMDP}. We state these two bounds in the following two lemmas, Lemma \ref{lemma:sp-A} and Lemma \ref{lemma:sp-D}.

\begin{lemma} \label{lemma:sp-A}
Suppose MDP $M$ is weakly communicating. For any $\varepsilon \in (0, 1]$ and $\gamma := 1 - \frac{\varepsilon}{H}$, it holds that $sp((1-\gamma)V_\gamma^*) \leq 4\varepsilon$.
\end{lemma}

\begin{proof}[Proof Idea]
We can rewrite the Bellman optimality equation of a DMDP into a form similar to that of an AMDP, 
\[ (\rho_\gamma^* + h_\gamma^*)(s) = \max_a \left\{ r(s,a) + \gamma P_{s,a}^\mathrm T h_\gamma^* \right\}, \]
by defining the values $\rho_\gamma^*$ and $h_\gamma^*$ properly. Comparing this equation and the equation of corresponding AMDP, we can establish an upper bound of $\| h_\gamma^* - h^* \|_\infty$, and then obtain an upper bound of $sp(V_\gamma^*)$ in $H$. See Appendix \ref{sec:appendix-upper} for details.
\end{proof}

\begin{lemma} \label{lemma:sp-D}
Suppose MDP $M$ is weakly communicating. For any $\varepsilon \in (0, 1]$ and $\gamma := 1 - \frac{\varepsilon}{H}$, it holds that $sp((1-\gamma)V_\gamma^{\pi^*}) \leq 4\varepsilon$.
\end{lemma}

\begin{proof}[Proof Idea]
The key fact used here is Lemma \ref{lemma:sp-F} in the appendix, which states that the span of FMDP values can always be bound by the span of AMDP values. We then acquire a control for all finite horizon values, no matter how long the horizon is. To analyse $V_\gamma^{\pi^*}$, we can write it into an infinite sum, truncate the infinite sum into a finite sum, cleverly rearrange the summation into a (weighted) sum of several finite horizon values, and finally push the truncation to infinity. See Appendix \ref{sec:appendix-upper} for details.
\end{proof}



With these lemmas at hand, we can now easily prove our Theorem \ref{thm:reduction}:

\begin{proof}[Proof of Theorem \ref{thm:reduction}]

Given any $\varepsilon \in (0, 1]$ and any $\delta \in (0, 1]$, we take $\gamma = 1 - \frac{\varepsilon}{H}$.

Suppose $\hat\pi$ is an $\varepsilon_\gamma$-optimal policy in DMDP $(M, \gamma)$, then
\[ \|(1-\gamma)V_\gamma^* - (1-\gamma)V_\gamma^{\hat\pi}\|_\infty \leq (1-\gamma) \varepsilon_\gamma. \]
By Lemma \ref{lemma:reduction}, 
\[ \|\rho^\pi - (1-\gamma)V_\gamma^\pi\|_\infty \leq sp((1-\gamma)V_\gamma^\pi). \]
By Lemma \ref{lemma:sp-A} and Lemma \ref{lemma:sp-D}, we have
\[ sp((1-\gamma)V_\gamma^{\pi^*}) \leq 4 \varepsilon, \]
and
\[ sp((1-\gamma)V_\gamma^*) \leq 4 \varepsilon. \]
Hence, 
\begin{align} sp((1-\gamma)V_\gamma^{\hat\pi}) &\leq sp((1-\gamma)V_\gamma^*) + 2 \|(1-\gamma)V_\gamma^* - (1-\gamma)V_\gamma^{\hat\pi}\|_\infty \notag\\
&\leq 4 \varepsilon + 2 (1-\gamma) \varepsilon_\gamma. \end{align}
Moreover, note that
\[ V_\gamma^{\pi^*} \leq V_\gamma^* \]
holds naturally because of the optimality of $V_\gamma^*$.

Merging the inequalities above,
\begin{align*}
\rho^* &\leq (1 - \gamma) V_\gamma^{\pi^*} + 4 \varepsilon + 2 (1-\gamma) \varepsilon_\gamma \\
	   &\leq (1 - \gamma) V_\gamma^* + 4 \varepsilon + 2 (1-\gamma) \varepsilon_\gamma \\
	   &\leq (1 - \gamma) V_\gamma^{\hat\pi} + 4 \varepsilon + 3 (1-\gamma) \varepsilon_\gamma \\
	   &\leq \rho^{\hat\pi} + 8 \varepsilon + 3 (1-\gamma) \varepsilon_\gamma,
\end{align*}
and the theorem has been proven. \qedhere

\end{proof}

Finally, plug in $\varepsilon$ in Theorem \ref{thm:reduction} with $\frac{\varepsilon}{12}$ and take $\varepsilon_\gamma = \frac{\varepsilon}{12(1-\gamma)}$, and then Theorem \ref{thm:upperbound} follows.

\subsection{Lower bound}
\label{sec:lowerbound}

Next, we prove Theorem \ref{thm:lowerbound} by considering a family of hard MDPs and then using an information theoretical analysis. First, we define the correctness of a RL algorithm.

\begin{definition} \label{def:correct}
($(\varepsilon, \delta)$-correctness) For any $\varepsilon \in (0, 1]$ and $\delta \in (0, 1]$, an AMDP algorithm $\mathscr A$ is called $(\varepsilon, \delta)$-correct if it promises to output an $\varepsilon$-optimal (stationary) policy with probability at least $1 - \delta$ for any communicating AMDP problem.
\end{definition}

To construct the final hard MDPs with parameters $S := |\mathcal S|, A := |\mathcal A|$ and diameter not larger than $D$, we need some components and some notation for brief. In the following, define $A^\prime = A-1$, $D^\prime = \frac{D}{8}$ and $K = \lceil \frac{S}{3} \rceil$. Assume $A \geq 3$, $D \geq \max\{16 \lceil \log_A S \rceil, 16\}$ and $\varepsilon \leq \frac{1}{16}$.

\paragraph{The core component MDP}
We first describe a component MDP with $2$ states $x, y$, $A^\prime$ actions and the parameter $D^\prime$ (not exactly its diameter). For each action $a$, let the probability transitioning from one state to another be $\frac{1 + 8\varepsilon}{D^\prime}$. The reward function is set to be $r(x, a) = 1$ and $r(y, a) = 0$ for all actions. This component MDP is depicted in Figure \ref{fig:component}.

\begin{figure}[!ht]
\centering
\caption{The core component MDP\label{fig:component}}
\begin{tikzpicture}
\fill (-2,0) circle (2pt);
\fill (2,0) circle (2pt);
\draw[->] (-2,0) arc (180:5:2 and 1.5);
\draw[->] (-2,0) arc (180:5:2 and 0.8);
\draw[->] (2,0) arc (0:-175:2 and 1.5);
\draw[->] (2,0) arc (0:-175:2 and 0.8);
\draw[->] (-2,0) arc (0:-340:0.5);
\draw[->] (2,0) arc (180:-160:0.5);
\node at (-1.6,0) {x};
\node at (1.6,0) {y};
\node at (0,1.25) {\vdots};
\node at (0,-1.05) {\vdots};
\node at (0,1.8) {$p=\nicefrac{1}{D}$};
\node at (0,-1.8) {$p=\nicefrac{1}{D}$};
\end{tikzpicture}
\end{figure}

\paragraph{An entire MDP $M_0$}
Then we connect $K$ copies of the component MDP into an entire MDP, $M_0$, on which we will further define the family of hard MDPs. It's easy to show that there always exists an $A^\prime$-ary tree with depth not larger than $\lceil \log_{A^\prime} S \rceil + 1$ having exactly $S-2K$ non-leaf nodes and $K$ leaves\footnote{We can first draw down the $K$ leaves and directly build an $A^\prime$-ary tree with not larger than $S-2K$ non-leaf nodes, and then insert one child for some old leaves.}. We define $M_0$ to be such a tree with every leaf substituted by a copy of the component MDP (the upper node is the $x$ state and the lower node is the $y$ state), as illustrated in Figure \ref{fig:M0}. For every state in this $A^\prime$-ary tree (including the $x$ states but not including the $y$ states), it has different actions deterministically going to its all children respectively, if it's not a leaf, and its parent state, if it's not the root. Any remaining action of these states is defined to be a self-circle action. All the actions defined above is deterministic and with reward $0$. The $K$ component MDPs have the same $A^\prime$ actions (WLOG the first $A^\prime$ of all $A$ actions) and rewards defined in last paragraph. The one remaining action of every $y$ state in each component MDP is also defined as a deterministic and no-reward self-circle. By the construction, it's obvious that $K \geq \frac{S}{3}$ and that $M_0$ has a diameter not larger than $2 (\frac{D^\prime}{1+8\varepsilon} + \log_{A^\prime}S + 1) \leq D$, since $D^\prime := \frac{D}{8}$ and $\log_A S \leq \frac{D}{8}$ by assumption.

\begin{figure}[!ht]
\centering
\caption{An example of $M_0$ when $A=4, S=14$ (actions ignored)\label{fig:M0}}
\begin{tikzpicture}
\coordinate (t1) at (0,0);
\coordinate (t2) at (-2,-1);
\coordinate (t3) at (2,-1);
\coordinate (t4) at (-3,-2);
\coordinate (x1) at (-3,-3);
\coordinate (x2) at (-2,-2);
\coordinate (x3) at (-1,-2);
\coordinate (x4) at (1.5,-2);
\coordinate (x5) at (2.5,-2);
\coordinate (y1) at (-3,-3.5);
\coordinate (y2) at (-2,-2.5);
\coordinate (y3) at (-1,-2.5);
\coordinate (y4) at (1.5,-2.5);
\coordinate (y5) at (2.5,-2.5);
\fill (t1) circle (2pt);
\fill (t2) circle (2pt);
\fill (t3) circle (2pt);
\fill (t4) circle (2pt);
\fill (x1) circle (2pt);
\fill (x2) circle (2pt);
\fill (x3) circle (2pt);
\fill (x4) circle (2pt);
\fill (x5) circle (2pt);
\fill (y1) circle (1pt);
\fill (y2) circle (1pt);
\fill (y3) circle (1pt);
\fill (y4) circle (1pt);
\fill (y5) circle (1pt);
\draw (t1) -- (t2);
\draw (t1) -- (t3);
\draw (t2) -- (t4);
\draw (t4) -- (x1);
\draw (t2) -- (x2);
\draw (t2) -- (x3);
\draw (t3) -- (x4);
\draw (t3) -- (x5);
\draw (x1) -- (y1);
\draw (x2) -- (y2);
\draw (x3) -- (y3);
\draw (x4) -- (y4);
\draw (x5) -- (y5);
\end{tikzpicture}
\end{figure}
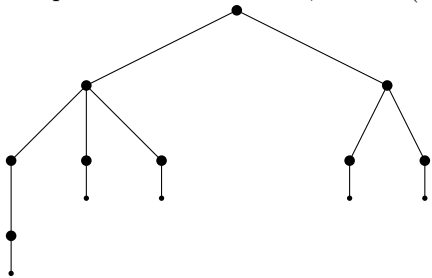

\paragraph{The hard MDPs: $M_1$ and $M_{k,l}$}
Now we can describe the hard MDPs used to prove the lower bound. $M_1$ is defined nearly the same as $M_0$, except the transition probability from every $x_k$ state when taking action $a_1$ is decreased from $\frac{1 + 8\varepsilon}{D^\prime}$ to $\frac{1}{D^\prime}$. A direct calculation shows that the only $\varepsilon$-optimal action at every $x_k$ state is $a_1$. For any $1 \leq k \leq K$ and any $2 \leq l \leq A^\prime$, $M_{k,l}$ is defined by further changing one transition probability from $M_1$. In $M_{k,l}$, only $p(y_k|x_k,a_l)$ is decreased from $\frac{1 + 8\varepsilon}{D^\prime}$ to $\frac{1 - 8\varepsilon}{D^\prime}$, making the only $\varepsilon$-optimal action at state $x_k$ to be $a_l$. It's obvious that these MDPs also have diameter not larger than $D$.

The intuition is that, since $M_1$ and $M_{k,l}$ disagree with the (unique) $\varepsilon$-optimal action at state $x_k$, any $(\varepsilon, \delta)$-correct algorithm must distinguish between $M_1$ and $M_{k, l}$, whose only difference is the transition probability at $(s_k, a_l)$. It can be shown, through an information theoretical analysis, that $\Omega(D\varepsilon^{-2}\ln\frac{1}{\delta})$ samples at this state-action pair is necessary when the algorithm is running in $M_1$. Sum up for all $k, l$ and then Theorem \ref{thm:lowerbound} follows. We follow a general proof structure as  the one in \citet{feng2019does}. The formal proof is presented in \ref{sec:appendix-lower}.

\section{Relationship Between Different Parameter Regimes}
\label{sec:mix}

In this section, we investigate the relationship between the following parameters (which were used in the literature to study the sample complexity of AMDP):
\begin{itemize}
	\item Diameter:
		\[ D := \underset{s_1 \neq s_2}{\max} \min_\pi \mathbb E_{s_1}^\pi [\tau_{s_2}], \]
		where $\tau_{s_2}$ is the hitting time to $s_2$;
	\item Mixing time: 
		\[ t_{mix} := \max_\pi \left[ \underset{t \geq 1}{\arg \min} \left\{ \underset{\mu_0}{\max} \| (\mu_0 P_\pi^t)^\mathrm T - \nu^\mathrm T \|_1 \leq \frac{1}{2} \right\} \right], \]
		where $\mu_0$ is any initial probability distribution on $\mathcal S$, $\nu$ is its invariant distribution, and the maximum of $\pi$ is taken under all deterministic stationary policies;
	\item Span of optimal bias:
		\[ H := sp(h^*), \]
		where $h^* := h^{\pi^*}$ and $\pi^*$ is a gain-optimal policy of the AMDP.
\end{itemize}

We will show the following facts about these parameters:
\begin{itemize}
	\item $D$ and $t_{mix}$ cannot bound each other;
	\item $D$ and $t_{mix}$ are both upper bounds of $H$.
\end{itemize}



\subsection{$D$ is not an upper bound of $t_{mix}$}

In this subsection, we show that $t_{mix}$ can be arbitrary larger than $D$. 

It's obvious that the definition of $t_{mix}$ requires the aperiodicity of the MDP, since in a periodic MDP $t_{mix} = +\infty$ and conclusions in \citet{jin2021towards} fails. This phenomenon can be seen even in the following simple MDP $M = (\mathcal S, \mathcal A, P, r)$: $\mathcal A = \{ a \}$, $P(s_{k+1} | s_k, a) = 1$ for all $k < |\mathcal S|$ and $P(s_1 | s_{|\mathcal S|}, a) = 1$, and $|\mathcal S| = D \geq 2$.

One may argue that the comparison is fair only when the assumptions of both conclusions hold, i.e., $D < +\infty$ and $t_{mix} < +\infty$. But a slight modification from $M$ shows that, even when the MDP is aperiodic and $t_{mix} < +\infty$, $t_{mix}$ can still be arbitrary larger than $D$.

Let $M_\tau = (\mathcal S, \mathcal A, P_\tau, r)$ be a $\tau$-aperiodicity-transformation (\citet[Sec 8.5.4]{puterman2014markov}) of $M$ for any $\tau \in (0, 1)$. That is, $P_\tau(s^\prime | s, a) := (1 - \tau) P(s^\prime | s, a)$ for any $s^\prime \neq s$, and $P_\tau(s | s, a) := (1 - \tau) P(s | s, a) + \tau$. Let $\tau \rightarrow 0^+$ and $\tau \leq \frac{1}{2}$, it's obvious that the diameter of $M_\tau$ remains bounded by $2D$, while $t_{mix}$ goes to $+\infty$ though being finite for any $\tau > 0$.

\subsection{$t_{mix}$ is not an upper bound of $D$}


Consider a simple MDP $M_0$ with only two states $x, y$ and one action $a$. The transition probabilities are set to $p(y|x,a) = \frac{1}{D}$ and $p(x|y,a) = 1$. A simple calculation shows that $t_{mix}(\varepsilon) = O(\log_D \frac{1}{\varepsilon})$. Hence, for any fixed accuracy $\varepsilon > 0$, $t_{mix}(\varepsilon)$ remains bounded as $D \rightarrow +\infty$.  Especially, $t_{mix} := t_{mix}(\frac{1}{2})$ is bounded as $D \rightarrow +\infty$.


\subsection{$D$ and $t_{mix}$ are both upper bounds of $H$}

It's well known (\citet{bartlett2012regal}) that $H \leq D$.

In this subsection, we show the other bound: $ H = O(t_{mix})$.

\begin{lemma} \label{lemma:sp-mix}
For any MDP with $t_{mix} < +\infty$ and any stationary policy $\pi$, for any recurrent block $C_k$ of MC $P_\pi$ and any positive integer $T$, it holds that $sp(V_{T,k}^\pi) \leq 4 t_{mix}$.
\end{lemma}

\begin{proof}
Define as in \citet{levin2017markov}:
\[ d(t) := \frac{1}{2} \max_s \| (e_s P_\pi^t)^\mathrm T - \nu^\mathrm T \|_1, \]
where $e_s$ denotes the row vector with $1$ at $s$ and $0$ at other entries, and $\nu$ is the invariant distribution.

It's obvious that $d(t)$ is decreasing in $t$. Moreover, \citet[Sec 4.5]{levin2017markov} shows that $d(l t_{mix}) \leq 2^{-l}$ for any $l \in \mathbb N$.

Then we can prove our lemma:
\begin{align*}
sp(V_{T,k}^\pi) &\leq \sum_{t=0}^{T-1} sp(P_{\pi,k}^t r_{\pi,k}) \leq 2 \sum_{t=0}^{T-1} d(t) = 2 \sum_{l=0}^{\frac{T}{t_{mix}} - 1} \sum_{t=0}^{t_{mix} - 1} d(l t_{mix} + t) \leq 2 \sum_{l=0}^{\frac{T}{t_{mix}} - 1} \sum_{t=0}^{t_{mix} - 1} d(l t_{mix})\\& \leq 2 \sum_{l=0}^{\frac{T}{t_{mix}} - 1} \sum_{t=0}^{t_{mix} - 1} 2^{-l} 
				\leq 4 t_{mix}.
\end{align*}
\end{proof}

With Lemma \ref{lemma:sp-mix}, we can easily show $H = O(t_{mix})$ then:

\begin{proposition}
For any MDP with $t_{mix} < +\infty$, it always holds that $H \leq 8 t_{mix}$.
\end{proposition}

\begin{proof}
Since $t_{mix} < +\infty$, we know the MDP is aperiodic.

Hence, 
\[ h^\pi = \underset{T \to \infty}{\lim} V_T^\pi - T \rho^\pi, \]
which implies with Lemma \ref{lemma:sp-mix} that
\[ sp(h^\pi_k) \leq 4 t_{mix}, \ \ \ \forall k. \]

Since in any recurrent block $C_k$, $h^\pi_k$ cannot be positive or negative on all entries, it holds that
\[ sp(h^\pi) \leq 2 \max_k \|h^\pi_k\|_\infty \leq 2 \max_k sp(h^\pi_k) \leq 8 t_{mix}. \]

Take $\pi = \pi^*$ and the proposition follows.
\end{proof}

\section{Conclusion}

In this work, we establish an $\widetilde O(|\mathcal S| |\mathcal A| H \varepsilon^{-3} \ln \frac{1}{\delta})$ sample complexity upper bound of the $\varepsilon$-optimal policy learning problem in AMDPs, and an $\Omega(|\mathcal S| |\mathcal A| D \varepsilon^{-2} \ln \frac{1}{\delta})$ lower bound (with a corallary of $\Omega(|\mathcal S| |\mathcal A| H \varepsilon^{-2} \ln \frac{1}{\delta})$ lower bound), matching in all parameters of $(|\mathcal S|, |\mathcal A|, H, \ln \frac{1}{\delta})$ up to some logarithmic factors. Our work also opens up a direction for further improvement:
 Our upper bound is one $\varepsilon$ looser than the lower bound. Though we believe that an upper bound of $\widetilde O(|\mathcal S| |\mathcal A| H \varepsilon^{-2} \ln \frac{1}{\delta})$ is attainable, it may be necessary to develop new methods to fix this gap of complexity, instead of directly reducing AMDPs to DMDPs.
	

\bibliographystyle{plainnat}
\bibliography{citations}

\newpage

\appendix

\onecolumn

\section{Proofs of the lemmas}
\label{sec:appendix-proof}

\subsection{Lemmas for the upper bound}
\label{sec:appendix-upper}

First, we directly prove Lemma \ref{lemma:reduction} and Lemma \ref{lemma:sp-A}.

\begin{proof}[Proof of Lemma \ref{lemma:reduction}]
This lemma follows directly from the closed-form expressions of $\rho^\pi$ and $V_\gamma^\pi$, in essence saying that, for any $s$, $\rho^\pi(s)$ is an average of $V_\gamma^\pi$, as shown below.

We have analytical expressions for $\rho^\pi$ and $V_\gamma^\pi$:
\[ \rho^\pi = P_\pi^* r_\pi, \]
\[ V_\gamma^\pi = (I - \gamma P_\pi)^{-1} r_\pi, \]
where $P_\pi$ is the probability matrix of the induced Markov chain of $\pi$, and $r_\pi$ is the induced reward function, $r_\pi(s) = \mathbb E[r(s,a) | a \sim \pi(s)]$.

Note that
\[ P_\pi^* (I-\gamma P_\pi) = P_\pi^* - \gamma P_\pi^* = (1-\gamma) P_\pi^*, \]
thus
\[ P_\pi^* = (1-\gamma) P_\pi^* (I-\gamma P_\pi)^{-1}. \]
Hence,
\[ \rho^\pi = (1-\gamma) P_\pi^* (I-\gamma P_\pi)^{-1} r_\pi = P_\pi^* \cdot (1-\gamma) V_\gamma^\pi. \]
Since in a weakly communicating MDP, $\rho^\pi(s)$ is always constant about $s$, we have
\[ \min_s (1-\gamma) V_\gamma^\pi(s) \leq \rho^\pi \leq \max_s (1-\gamma) V_\gamma^\pi(s), \]
which implies that
\[ \|\rho^\pi - (1-\gamma)V_\gamma^\pi\|_\infty \leq sp((1-\gamma)V_\gamma^\pi). \]
\end{proof}

\begin{proof}[Proof of Lemma \ref{lemma:sp-A}]

When the MDP $M$ is weakly communicating, it's known that $\rho^*$ is a constant vector and that $\rho^*$ and $h^*$ satisfy the Bellman optimality equation
\[ (\rho^* + h^*)(s) = \max_a \left\{ r(s,a) + P_{s,a}^\mathrm T h^* \right\}. \]

It's also known that, for any $\gamma \in [0,1)$, $V_\gamma^*$ satisfies the discounted Bellman equation
\[ V_\gamma^*(s) = \max_a \left\{ r(s,a) + \gamma P_{s,a}^\mathrm T V_\gamma^* \right\}. \]

To rewrite the discounted Bellman equation into a form similar to the AMDP one, we define
\[ h_\gamma^* = V_\gamma^* - \frac{1}{1-\gamma} \rho^*. \]
Note that $h_\gamma^*$ and $V_\gamma^*$ have the same span since $\rho^*$ is a constant vector.

A direct check shows that $h_\gamma^*$ satisfies the following optimality equation:
\[ (\rho^* + h_\gamma^*)(s) = \max_a \left\{ r(s,a) + \gamma P_{s,a}^\mathrm T h_\gamma^* \right\}. \]

Putting this equation and the AMDP Bellman equation together, we can then analysis a bound for $\|h^* - h_\gamma^*\|_\infty$. For all $s$,
\begin{align*}
|(h^* - h_\gamma^*)(s)| &= |\max_a \left\{ r(s,a) + P_{s,a}^\mathrm T h^* \right\} - \max_a \left\{ r(s,a) + \gamma P_{s,a}^\mathrm T h_\gamma^* \right\}| \\
&\leq |\max_a \left\{ P_{s,a}^\mathrm T h^* - \gamma P_{s,a}^\mathrm T h_\gamma^* \right\}| \\
&\leq |\max_a \left\{ \gamma P_{s,a}^\mathrm T (h^* - h_\gamma^*) \right\}| + |\max_a \left\{ (1-\gamma) P_{s,a}^\mathrm T h^* \right\}| \\
&\leq \gamma \|h^* - h_\gamma^*\|_\infty + (1-\gamma) \|h^*\|_\infty,
\end{align*}
and it follows that
\[ \|h^* - h_\gamma^*\|_\infty \leq \|h^*\|_\infty \]
and thus
\[ sp(V_\gamma^*) = sp(h_\gamma^*) \leq 2 \|h_\gamma^*\|_\infty \leq 4 \|h^*\|_\infty \leq 4 sp(h^*) = 4H. \]

\end{proof}

To prove Lemma \ref{lemma:sp-D}, we need another auxiliary lemma:

\begin{lemma} \label{lemma:sp-F}
For any weakly communicating MDP $M$, any optimal policy $\pi^*$ of its AMDP problem and any positive integer $T$, it holds that $sp(V_T^{\pi^*}) \leq 2 sp(h^{\pi^*})$.
\end{lemma}

\begin{proof}
First note that in a weakly communicating, $\rho^{\pi^*}$ is a constant vector. Below denote $\pi = \pi^*$. (to note that this lemma actually holds for any $\pi$ with a constant $\rho^\pi$)

Since
\begin{align*}
h^\pi &= \underset{t \rightarrow \infty}{C\text{-}\lim} (V_t^\pi - t \rho^\pi) \\
		&= \underset{N \rightarrow \infty}{\lim} \frac{1}{N} \sum_{t=1}^N (V_t^\pi - t \rho^\pi) \\
		&= \underset{N \rightarrow \infty}{\lim} \frac{1}{N} \sum_{t=T+1}^N (V_t^\pi - t \rho^\pi) \\
		&= \underset{N \rightarrow \infty}{\lim} \frac{1}{N} \sum_{t=T+1}^N [(V_T^\pi - T \rho^\pi) + P_\pi^T (V_{t-T}^\pi - (t-T) \rho^\pi)] \\
		&= \underset{N \rightarrow \infty}{\lim} \frac{1}{N} \sum_{t=T+1}^N (V_T^\pi - T \rho^\pi) + \underset{N \rightarrow \infty}{\lim} \frac{1}{N} \sum_{t=T+1}^N P_\pi^T (V_{t-T}^\pi - (t-T) \rho^\pi) \\
		&= (V_T^\pi - T \rho^\pi) + P_\pi^T \underset{N \rightarrow \infty}{\lim} \frac{1}{N} \sum_{t=1}^{N-T} (V_t^\pi - t \rho^\pi) \\
		&= (V_T^\pi - T \rho^\pi) + P_\pi^T h^\pi,
\end{align*}
it follows that
\[ V_T^\pi = T \rho^\pi + h^\pi - P_\pi^T h^\pi \]
and then
\[ sp(V_T^\pi) \leq 2 sp(h^\pi). \]
\end{proof}

Now we can prove Lemma \ref{lemma:sp-D}:

\begin{proof}[Proof of Lemma \ref{lemma:sp-D}]
Since
\begin{align*}
V_\gamma^{\pi^*} &= \lim_{T \rightarrow \infty} \sum_{t=0}^{T-1} \gamma^t P_{\pi^*}^t r_{\pi^*} \\
				 &= \lim_{T \rightarrow \infty} V_T^{\pi^*} - (1-\gamma) \sum_{t=1}^{T-1} \gamma^{t-1} P_{\pi^*}^t V_{T-t}^{\pi^*},
\end{align*}
we have
\begin{align*}
sp(V_\gamma^{\pi^*}) &\leq \underset{T \rightarrow \infty}{\lim\sup\ } sp(V_T^{\pi^*}) + (1-\gamma) \sum_{t=1}^{T-1} \gamma^{t-1} sp(V_{T-t}^{\pi^*}) \\
					 &\leq 2 sp(h^{\pi^*}) \left( 1 + (1-\gamma) \sum_{t=1}^\infty \gamma^{t-1} \right) \\
					 &= 4 sp(h^{\pi^*}) \\
					 &= 4H,
\end{align*}
and Lemma \ref{lemma:sp-D} follows.
\end{proof}

\subsection{Lemmas for the lower bound}
\label{sec:appendix-lower}

We refer the construction of hard MDPs $M_1$ and $M_{k, l}, 1 \leq k \leq K, 2 \leq l \leq A^\prime$ to Section \ref{sec:lowerbound}.
Remind that these MDPs all have parameters $|\mathcal S| = S, |\mathcal A| = A$ and diameter not larger than the given $D$. Also remind we have defined some notation: $A^\prime = A - 1, D^\prime = \frac{D}{8}$ and the number of the 2-state component in the MDP $K = \lceil \frac{S}{3} \rceil$.
We assume $\varepsilon \leq \frac{1}{32}, \delta \leq \frac{1}{16}, A \geq 3$, and $D \geq \max \{ 16 \lceil \log_A{S} \rceil, 16 \}$.

Before stepping into the analysis, we first introduce some another notation:

Assume algorithm $\mathscr A$ is $(\varepsilon, \delta)$-correct on all these MDPs. Denote by $N_{k, l}$ the number of samples $\mathscr A$ takes on $(x_k, a_l)$, which is a random variable with randomness on both the generative model of the underground MDP and $\mathscr A$ itself (to allow random algorithms). Denote $\mathbb P_1, \mathbb E_1, \mathbb P_{k,l}, \mathbb E_{k,l}$ as the probability and expectation when the underground MDP is $M_1$ and $M_{k,l}$, respectively. We aim to prove a lower bound of \[ LB := C d \epsilon^{-2} \ln\frac{1}{4 \delta} = \Theta(D \varepsilon^{-2} \ln \frac{1}{\delta}) \] for $\mathbb E_1 [N_{k,l}], \forall k \forall l$, where $C > 0$ is a universal constant to be determined later, and $d := D^\prime = \frac{D}{8}$ and $\epsilon := 8 \varepsilon$ just to simplify the notation below.

We define the following events:
\[ A_{k,l} := \left\{ N_{k,l} \leq 4 LB \right\},\ \ \ \ \ E_{k} := \left\{ \mathscr A \text{\ outputs a policy\ } \pi \text{\ with\ } \pi(s_k) = a_1 \right\}, \]
\[ C_{k,l} := \left\{ \underset{1 \leq N_{k,l} \leq 4 LB}{\max} | p_0 \cdot N_{k,l} - S_{k,l}(N_{k,l}) | \leq \sqrt{16 LB \cdot p_0 \cdot (1-p_0) \cdot \ln \frac{1}{4\delta} } \right\}, \]
where $S_{k,l}(N_{k,l})$ is the times the generative model returns $x_k$ as the next state when $\mathscr A$ calls it at $(x_k, a_l)$ for $N_{k,l}$ times, and $p_0 = 1 - \frac{1+8\varepsilon}{D^\prime}$.

The following two lemmas are exactly same as the Lemma 7 and Lemma 8 in \citet{feng2019does}. We state these lemmas and give a proof here for completeness.

\begin{lemma} \label{lemma:pr-event-A}
For any $1 \leq k \leq K$ and $2 \leq l \leq A^\prime$, if $\mathbb E_1 [N_{k,l}] \leq LB$, then $\mathbb P_1 (A_{k,l}) \geq \frac{3}{4}$.
\end{lemma}

\begin{proof}
By Markov inequality, 
\[ \mathbb P_1 (A_{k,l}) = 1 - \mathbb P_1 (\overline {A_{k,l}}) \geq 1 - \frac{\mathbb E_1 [N_{k,l}]}{4 LB} \geq \frac{3}{4}. \]
\end{proof}

\begin{lemma} \label{lemma:pr-event-C}
For any $1 \leq k \leq K$ and $2 \leq l \leq A^\prime$, $\mathbb P_1 (C_{k,l}) \geq \frac{3}{4}$.
\end{lemma}

\begin{proof}
By definition of $S_{k,l}$ and the MDPs, $S_{k,l}$ is the sum of $N_{k,l}$ i.i.d. Bernoulli$\left( 1 - \frac{1+8\varepsilon}{D^\prime} \right)$ random variables, and $\left( 1 - \frac{1+8\varepsilon}{D^\prime} \right) \cdot N_{k,l} - S_{k,l}(N_{k,l})$ is a martingale. Using Doob's inequality (\citet[Thm 4.4.2]{durrett2019probability})
, we have
\begin{align*}
\mathbb P_1(C_{k,l}) &= 1 - \mathbb P_1 (\overline {C_{k,l}}) \\
					 &\geq 1 - \frac{\mathbb E_1 \left[ \left( \left( 1 - \frac{1+8\varepsilon}{D^\prime} \right) \cdot 4 LB - S_{k,l}(4 LB) \right) ^2 \right]}{16 LB \cdot \frac{1+8\varepsilon}{D^\prime} \cdot \left( 1 - \frac{1+8\varepsilon}{D^\prime} \right) \cdot \ln \frac{1}{4\delta}} \\
					 &= 1 - \frac{1}{4 \ln \frac{1}{4\delta}} \\
					 &\geq \frac{3}{4}
\end{align*}
given that $\delta \leq \frac{1}{16}$.
\end{proof}

Now we can prove the following crucial lemma by a transformation between the probability measures.

\begin{lemma} \label{lemma:lowerbound-core}
For any $1 \leq k \leq K$ and $2 \leq l \leq A^\prime$, if $\mathbb E_1 [N_{k,l}] \leq LB$ and $\mathbb P_1(E_{k}) \geq 1 - \delta$, then $\mathbb P_{k,l} (E_{k}) > \delta$.
\end{lemma}

\begin{proof}
First, define an event
\[ \mathcal E_{k,l} := A_{k,l} \cap E_{k} \cap C_{k,l}. \]
Since
\[ \mathbb P_1 (A_{k,l}) \geq \frac{3}{4} \]
and
\[ \mathbb P_1 (C_{k,l}) \geq \frac{3}{4} \]
by Lemma \ref{lemma:pr-event-A} and Lemma \ref{lemma:pr-event-C}, and that
\[ \mathbb P_1 (E_{k}) \geq 1 - \delta > \frac{3}{4}, \]
we have
\[ \mathbb P_1 (\mathcal E_{k,l}) > \frac{1}{4}. \]
Define $W$ as the sequence of transitions when $\mathscr A$ calls the generative model $M$ for $N_{k,l}$ times. When $M = M_1$, the count of $x_k$ in this sequences, denoted as $S_{k,l}$, obeys a binomial distribution $B\left( 1 - \frac{1+8\varepsilon}{D^\prime}, N_{k,l} \right)$.

We define the likelihood functions $L_1$ and $L_{k,l}$ as
\[ L_1(w) = \mathbb P_1 (W = w), \ \ \ L_{k,l}(w) = \mathbb P_{k,l} (W = w). \]
Then we can use the method of measure transformation as following:
\begin{align*}
\mathbb P_{k,l}(E_k) &= \mathbb E_{k,l} \left[ I_{E_{k}} \right] \\
					 &= \mathbb E_1 \left[ \frac{L_{k,l}(W)}{L_1(W)} \cdot I_{E_{k}}(W) \right] \\
					 &\geq \mathbb E_1 \left[ \frac{L_{k,l}(W)}{L_1(W)} \cdot I_{\mathcal E_{k,l}}(W) \right] \\
					 &\geq \underset{w \in \mathcal E_{k,l}}{\min} \left\{ \frac{L_{k,l}(w)}{L_1(w)} \right\} \cdot \mathbb P_1(\mathcal E_{k,l}) \\
					 &> \underset{w \in \mathcal E_{k,l}}{\min} \left\{ \frac{L_{k,l}(w)}{L_1(w)} \right\} \cdot \frac{1}{4}.
\end{align*}

Remaining is to show that
\[ \frac{L_{k,l}(w)}{L_1(w)} \geq 4 \delta, \ \ \ \forall w \in \mathcal E_{k,l}. \]
Below we igonre the subscripts $(k, l)$ for $N_{k,l}$ and $S_{k,l}$ when it doesn't cause confusion, and denote $d := D^\prime$ and $\epsilon = 8 \varepsilon$, for simplicity of notation.

A direct calculation shows that
\begin{align*}
\frac{L_{k,l}(w)}{L_1(w)} &= \frac{\left( 1 - \frac{1-\epsilon}{d} \right) ^S \left( \frac{1-\epsilon}{d} \right) ^{N-S}}{\left( 1 - \frac{1+\epsilon}{d} \right) ^S \left( \frac{1+\epsilon}{d} \right) ^{N-S}} \\
						  &= \left( 1 + \frac{2d^{-1}\epsilon}{1-d^{-1}-d^{-1}\epsilon} \right) ^S \left( 1 - \frac{2\epsilon}{1+\epsilon} \right) ^{N-S} \\
						  &= \left( 1 + \frac{2d^{-1}\epsilon}{1-d^{-1}-d^{-1}\epsilon} \right) ^S \left( 1 - \frac{2\epsilon}{1+\epsilon} \right) ^{S \cdot \frac{d^{-1}+d^{-1}\epsilon}{1-d^{-1}-d^{-1}\epsilon}} \left( 1 - \frac{2\epsilon}{1+\epsilon} \right) ^{N - \frac{S}{1-d^{-1}-d^{-1}\epsilon}}.
\end{align*}
By our choice of $D$ and $\varepsilon$, it holds that $\frac{2\epsilon}{1+\epsilon} \in [0, \frac{1}{2}]$ and that $\frac{2d^{-1}\epsilon}{1-d^{-1}-d^{-1}\epsilon} + \frac{4d^{-1}\epsilon^2}{(1-d^{-1}-d^{-1}\epsilon)(1+\epsilon)} \in [0, 1]$. With the fact that $\ln (1-u) \geq -u-u^2, \forall u \in [0, \frac{1}{2}]$ and $e^{-u} \geq 1-u+\frac{u^2}{3}, \forall u \in [0, 1]$, we have
\begin{align*}
\left( 1 - \frac{2\epsilon}{1+\epsilon} \right) ^{\frac{d^{-1}+d^{-1}\epsilon}{1-d^{-1}-d^{-1}\epsilon}}
	&= \exp \left\{ \frac{d^{-1}+d^{-1}\epsilon}{1-d^{-1}-d^{-1}\epsilon} \cdot \ln \left( 1 - \frac{2\epsilon}{1+\epsilon} \right) \right\} \\
	&\geq \exp \left\{ \frac{d^{-1}+d^{-1}\epsilon}{1-d^{-1}-d^{-1}\epsilon} \cdot \left[ - \frac{2\epsilon}{1+\epsilon} - \left( \frac{2\epsilon}{1+\epsilon} \right)^2 \right] \right\} \\
	&= \exp \left\{ - \frac{2d^{-1}\epsilon}{1-d^{-1}-d^{-1}\epsilon} - \frac{4d^{-1}\epsilon^2}{(1-d^{-1}-d^{-1}\epsilon)(1+\epsilon)} \right\} \\
	&\geq 1 - \frac{2d^{-1}\epsilon}{1-d^{-1}-d^{-1}\epsilon} - \frac{4d^{-1}\epsilon^2}{(1-d^{-1}-d^{-1}\epsilon)(1+\epsilon)} \\
	 	&\ \ \ \ \ \ \ + \frac{1}{3} \left[ \frac{2d^{-1}\epsilon}{1-d^{-1}-d^{-1}\epsilon} + \frac{4d^{-1}\epsilon^2}{(1-d^{-1}-d^{-1}\epsilon)(1+\epsilon)} \right]^2 \\
	&\geq \left( 1 - \frac{2d^{-1}\epsilon}{1-d^{-1}-d^{-1}\epsilon} \right) \left( 1 - \frac{4d^{-1}\epsilon^2}{(1-d^{-1}-d^{-1}\epsilon)(1+\epsilon)} \right).
\end{align*}
Thus, 
\[ \frac{L_{k,l}(w)}{L_1(w)} \geq T_1 \cdot T_2 \cdot T_3, \]
where
\[ T_1 := \left( 1 - \frac{4d^{-2}\epsilon^2}{(1-d^{-1}-d^{-1}\epsilon)^2} \right) ^S, \]
\[ T_2 := \left( 1 - \frac{4d^{-1}\epsilon^2}{(1-d^{-1}-d^{-1}\epsilon)(1+\epsilon)} \right) ^S \]
and
\[ T_3 := \left( 1 - \frac{2\epsilon}{1+\epsilon} \right) ^{N - \frac{S}{1-d^{-1}-d^{-1}\epsilon}} \]
are three terms to be bounded.

Then we assume the event $\mathcal E_{k,l}$ happens, and prove lower bounds for these three terms.

As for the first term, since that $\ln(1-u) \geq -2u, \forall u \in [0, \frac{1}{2}]$ and that $\frac{4d^{-2}\epsilon^2}{(1-d^{-1}-d^{-1}\epsilon)^2} \in [0, \frac{1}{2}]$, we have
\begin{align*}
T_1 &\geq \left( 1 - \frac{4d^{-2}\epsilon^2}{(1-d^{-1}-d^{-1}\epsilon)^2} \right) ^{4 LB} \\
	&= \exp \left\{ 4 LB \cdot \ln \left( 1 - \frac{4d^{-2}\epsilon^2}{(1-d^{-1}-d^{-1}\epsilon)^2} \right) \right\} \\
	&\geq \exp \left\{ 4 C d \epsilon^{-2} \ln \frac{1}{4\delta} \cdot (-2) \cdot \frac{4d^{-2}\epsilon^2}{(1-d^{-1}-d^{-1}\epsilon)^2} \right\} \\
	&= \exp \left\{ -32 C \ln \frac{1}{4\delta} \cdot \frac{d^{-1}}{(1-d^{-1}-d^{-1}\epsilon)^2} \right\} \\
	&\geq \exp \left\{ -128 C \ln \frac{1}{4\delta} \right\} \\
	&= (4 \delta)^{128C}.
\end{align*}

As for the second term, since $\frac{4d^{-1}\epsilon^2}{(1-d^{-1}-d^{-1}\epsilon)(1+\epsilon)} \in [0, \frac{1}{2}]$, we have
\begin{align*}
T_2 &\geq \left( 1 - \frac{4d^{-1}\epsilon^2}{(1-d^{-1}-d^{-1}\epsilon)(1+\epsilon)} \right) ^{4 LB} \\
	&= \exp \left\{ 4 LB \cdot \ln \left( 1 - \frac{4d^{-1}\epsilon^2}{(1-d^{-1}-d^{-1}\epsilon)(1+\epsilon)} \right) \right\} \\
	&\geq \exp \left\{ 4 C d \epsilon^{-2} \ln \frac{1}{4\delta} \cdot (-2) \cdot \frac{4d^{-1}\epsilon^2}{(1-d^{-1}-d^{-1}\epsilon)(1+\epsilon)} \right\} \\
	&= \exp \left\{ -32 C \ln \frac{1}{4\delta} \cdot \frac{1}{(1-d^{-1}-d^{-1}\epsilon)(1+\epsilon)} \right\} \\
	&\geq \exp \left\{ -128 C \ln \frac{1}{4\delta} \right\} \\
	&= (4 \delta)^{128C}.
\end{align*}

Lastly, for the third term, since $\frac{2\epsilon}{1+\epsilon} \in [0, \frac{1}{2}]$, we have
\begin{align*}
T_3 &\geq \exp \left\{ \sqrt{16 LB \cdot \frac{d^{-1}+d^{-1}\epsilon}{1-d^{-1}-d^{-1}\epsilon} \ln \frac{1}{4\delta}} \cdot \ln \left( 1 - \frac{2\epsilon}{1+\epsilon} \right) \right\} \\
	&\geq \exp \left\{ \sqrt{16 C d \epsilon^{-2} \ln \frac{1}{4\delta} \cdot \frac{d^{-1}+d^{-1}\epsilon}{1-d^{-1}-d^{-1}\epsilon} \ln \frac{1}{4\delta}} \cdot (-2) \cdot \frac{2\epsilon}{1+\epsilon} \right\} \\
	&= \exp \left\{ -16\sqrt{C} \ln \frac{1}{4\delta} \cdot \sqrt{\frac{1}{(1-d^{-1}-d^{-1}\epsilon)(1+\epsilon)}} \right\} \\
	&\geq \exp \left\{ -32\sqrt{C} \ln \frac{1}{4\delta} \right\} \\
	&= (4 \delta)^{32\sqrt{C}}.
\end{align*}

Put these three bounds together and take $C = \frac{1}{2048}$, we finally get
\[ \frac{L_{k,l}(w)}{L_1(w)} \geq (4 \delta)^{128C + 128C + 32\sqrt{C}} \geq 4 \delta, \]
and it has been proven that
\[ \mathbb P_{k,l} (E_k) > \delta. \]

\end{proof}

The lower bound Theorem \ref{thm:lowerbound} follows directly from Lemma \ref{lemma:lowerbound-core}:

\begin{proof}[Proof of Theorem \ref{thm:lowerbound}]
Since algorithm $\mathscr A$ promises to be $(\varepsilon, \delta)$-correct, it must hold that
\[ \mathbb P_1 (E_k) \geq 1 - \delta, \ \ \ \mathbb P_{k,l} (E_k) \leq \delta, \ \ \ \forall 1 \leq k \leq K, \ \ \ \forall 2 \leq l \leq A^\prime, \]
because the only $\varepsilon$-optimal action at state $x_k$ is $a_1$ in $M_1$ and $a_l$ in $M_{k,l}$ (a direct calculation shows that).

Then by Lemma \ref{lemma:lowerbound-core}, the number of samples must satisfy $\mathbb E_1 [N_{k,l}] > LB$ for any $1 \leq k \leq K$ and any $2 \leq l \leq A^\prime$.

Denote by $N$ the total samples $\mathscr A$ call for the generative model on all state-action pairs, then
\begin{align*}
\mathbb E_1 [N] &\geq \sum_{k,l} \mathbb E_1 [N_{k,l}] \\
				&\geq K (A^\prime - 1) \cdot C d \epsilon^{-2} \ln \frac{1}{4\delta} \\
				&\geq \frac{S}{3} \frac{A}{3} \cdot C \cdot \frac{D}{8} (8 \varepsilon)^{-2} \cdot \frac{1}{2} \ln \frac{1}{\delta} \\
				&\geq c SAD \varepsilon^{-2} \ln \frac{1}{\delta},
\end{align*}
where we take $c = 2^{-25}$.
\end{proof}
\end{document}